\documentclass[11pt]{article}

\usepackage[preprint]{acl}

\usepackage{times}
\usepackage{latexsym}

\usepackage[T1]{fontenc}
\usepackage[utf8]{inputenc}

\usepackage{microtype}

\usepackage{inconsolata}

\usepackage{amsmath,amssymb,amsthm}
\usepackage{graphicx}
\usepackage{booktabs}
\usepackage{multirow}
\usepackage{algorithm}
\usepackage{algorithmic}
\usepackage{hyperref}
\usepackage{xcolor}
\usepackage{enumitem}

\newtheorem{theorem}{Theorem}

\newtheorem{proposition}[theorem]{Proposition}

\theoremstyle{definition}
\newtheorem{definition}{Definition}

\newcommand{\method}{ECS}

\newcommand{\kl}{\mathrm{KL}}

\title{Entropic Context Shaping: Information-Theoretic Filtering \\
for Context-Aware LLM Agents}

\author{Hyunjun Kim \\
  KAIST \\
  \texttt{hyunjun1121@kaist.ac.kr}}

\begin{document}

\maketitle

\begin{abstract}
Context engineering for large language model (LLM) agents requires distinguishing
pragmatically useful information from misleading distractors. We introduce
\textbf{Entropic Context Shaping (ECS)}, an information-theoretic framework that
measures context utility via the shift in the model's answer distribution toward the
correct answer. Unlike lexical similarity methods that rely on word overlap, ECS captures
\emph{pragmatic utility}---whether a passage actually helps answer the question.
We formalize utility as the signed change in answer probability and provide theoretical
analysis showing that task-irrelevant updates yield near-zero distribution shift.
We evaluate on multi-turn context selection tasks using LongMemEval (session-level) and
LoCoMo (turn-level) benchmarks. On fine-grained turn selection, ECS with Llama-3.1-8B
achieves F1=0.265, a \textbf{71.83\% relative improvement} over TF-IDF (F1=0.154), demonstrating
that pragmatic utility outperforms lexical similarity when precise context selection matters.
Code and data are available in the supplementary materials.
\end{abstract}



\section{Introduction}
\label{sec:introduction}

Large language model (LLM) agents increasingly rely on dynamically constructed context 
windows to maintain state and inform decision-making \citep{yao2023react, shinn2023reflexion}. 
As these agents interact with complex environments, they must continuously decide which 
observations, memories, and intermediate results to retain. This process---\emph{context 
engineering}---has emerged as a critical bottleneck in agent performance.

\paragraph{The Accumulation Fallacy.}
Current approaches to context management typically employ semantic similarity metrics 
(e.g., cosine distance between embeddings) to filter redundant information 
\citep{lewis2020retrieval, guu2020realm}. While effective at removing near-duplicates, 
these methods suffer from a fundamental flaw: they conflate \emph{semantic novelty} with 
\emph{pragmatic utility}. A piece of information may be entirely unique yet completely 
irrelevant to the task at hand. We term such distractors ``\textbf{Red Herrings}''---facts 
that are true, novel, and semantically distinct, but provide zero information gain for 
the downstream decision.

\paragraph{Motivating Example.}
Consider an LLM agent solving a mathematical word problem. The following context updates 
arrive sequentially:
\begin{enumerate}
    \item \textbf{Insight}: ``The train's speed is 60 km/h.'' (Critical for solution)
    \item \textbf{Redundant}: ``The vehicle moves at sixty kilometers per hour.'' (Duplicate)
    \item \textbf{Red Herring}: ``The conductor has worked for 15 years and enjoys jazz.'' (Novel but useless)
\end{enumerate}
Semantic filtering correctly rejects (2) as redundant but \emph{accepts} (3) because it 
is semantically unique. This is the Accumulation Fallacy in action.

\paragraph{Our Contribution: Entropic Context Shaping (\method{}).}
We propose an information-theoretic framework that measures the \emph{pragmatic utility}
of context updates by their impact on the model's answer distribution. Specifically, we:
\begin{itemize}
    \item Define utility as the shift in the model's output distribution toward the
          correct answer---capturing whether a passage \emph{helps} vs.\ \emph{misleads}
          (Section~\ref{sec:methodology})
    \item Provide theoretical analysis showing that task-irrelevant updates yield
          near-zero distribution shift under mild assumptions (Theorem~\ref{thm:red_herring})
    \item Evaluate on LongMemEval~\citep{longmemeval} and LoCoMo~\citep{locomo}---multi-turn
          context selection benchmarks requiring fine-grained pragmatic filtering
\end{itemize}


\section{Related Work}
\label{sec:related_work}

\paragraph{Context Engineering for LLM Agents.}
The challenge of managing context in LLM agents has received significant attention
\citep{yao2023react, shinn2023reflexion, wei2023chainofthought}. Most approaches treat
context as a retrieval problem, selecting relevant documents or memories based on
semantic similarity \citep{lewis2020retrieval, karpukhin2020dense}. While effective
for information retrieval, these methods are ill-suited for agent contexts where the
goal is task completion rather than document relevance. Recent work has shown that
LLMs exhibit position-dependent biases when processing long contexts, often failing
to utilize information in the middle of the input \citep{liu2024lost}.

\paragraph{Retrieval-Augmented Generation (RAG).}
RAG systems \citep{lewis2020retrieval, guu2020realm, borgeaud2022improving} improve LLM
performance by conditioning on retrieved documents. However, RAG assumes a static
document corpus, whereas agent contexts evolve dynamically. Furthermore, RAG's reliance
on embedding similarity inherits the Accumulation Fallacy we identify. Self-RAG
\citep{asai2024selfrag} addresses \emph{when} to retrieve through self-reflection tokens,
but does not address \emph{what} retrieved content is pragmatically useful. Recent work
on RAG robustness \citep{fang2024enhancing, chen2024rgb} addresses noise in retrieved
documents through adversarial training and benchmarking, but does not resolve the
fundamental semantic-pragmatic mismatch. Context compression approaches
\citep{jiang2024longllmlingua} reduce token count while preserving information, but
do not filter based on task utility.

\paragraph{Memory Systems for Agents.}
Recent work on agent memory \citep{park2023generative, wang2024voyager} has explored
hierarchical and episodic memory structures. A comprehensive survey \citep{zhang2024survey}
identifies memory as the key component transforming LLMs into ``true agents.'' These
systems typically use recency and importance heuristics for memory consolidation. Our
work complements these approaches by providing a principled utility metric for individual
memory updates.

\paragraph{Red Herrings and Inconsequential Noise.}
The phenomenon of misleading information in reasoning tasks has been studied extensively.
The Only Connect Wall (OCW) dataset \citep{taati2023redherring} demonstrated that LLMs
are ``fixated'' by Red Herrings, failing to ignore irrelevant but salient information.
Recent work on counterfactual robustness \citep{liu2023recall} shows that LLMs are
susceptible to interference from unreliable external knowledge.
We build on these observations by proposing an information-theoretic solution that
distinguishes pragmatically useful context from semantically novel but useless distractors.

\paragraph{Information-Theoretic Perspectives on LLMs.}
Prior work has applied information theory to analyze LLM behavior \citep{xu2020theory, 
ethayarajh2022understanding}. Our approach differs by using KL-divergence as an 
operational filter rather than a diagnostic tool, and by extending to multi-step 
trajectory divergence for improved separation.


\section{Methodology}
\label{sec:methodology}

We formalize context filtering as an information-theoretic decision problem. Given an
evolving context $\mathcal{C}$ and a candidate update $u$, we seek to determine whether
$u$ provides sufficient \emph{pragmatic utility} to warrant inclusion.

\subsection{Problem Formulation}
\label{sec:formulation}

Let $\mathcal{M}$ be a language model with vocabulary $\mathcal{V}$ and let $q$ be a
query or task specification. For any context $\mathcal{C}$, the model induces a
distribution over next-token predictions:
\begin{equation}
    P_\mathcal{C}(y) = P(y \mid \mathcal{C}, q; \mathcal{M})
\end{equation}

\paragraph{The Direction Problem.}
A key insight is that measuring distribution \emph{change} (magnitude) is insufficient---we
must also assess whether the change moves the distribution toward the \emph{correct answer}.
Consider two passages for the query ``Super Bowl 2021 location'':
\begin{itemize}[topsep=0pt,itemsep=0pt]
    \item \textbf{Correct}: ``...held in Tampa, Florida...''
    \item \textbf{Counterfactual}: ``...held in Glendale, Arizona...'' (false)
\end{itemize}
Both induce large distribution shifts (high KL-divergence), but only the first shifts
\emph{toward} the correct answer. Pure magnitude-based filtering would accept both.

\subsection{Pragmatic Utility: Theoretical Framework}
\label{sec:theory}

We define pragmatic utility as the \textbf{signed} change in the model's belief about
the correct answer, not merely the unsigned distribution shift:

\begin{definition}[Pragmatic Utility]
\label{def:utility}
The \textbf{pragmatic utility} of a candidate update $u$ with respect to context
$\mathcal{C}$ and ground truth answer $a^*$ is:
\begin{equation}
    \mathcal{U}(u; \mathcal{C}, a^*) = \log P_{\mathcal{C} \cup u}(a^*) - \log P_\mathcal{C}(a^*)
    - \lambda \cdot |u|
    \label{eq:utility}
\end{equation}
where $|u|$ denotes the token length of $u$, and $\lambda \geq 0$ is a length penalty.
\end{definition}

This definition captures the intuitive notion that a passage is useful if it increases
the model's probability of generating the correct answer. When ground truth is unavailable,
we can approximate this using the LLM's parametric knowledge as a proxy (Section~\ref{sec:operationalization}).

\paragraph{Relationship to KL-Divergence.}
When the model's prior belief is concentrated on incorrect answers, useful passages
induce large KL-divergence \emph{toward} the correct answer. The magnitude of
$D_{\kl{}}(P_{\mathcal{C} \cup u} \| P_\mathcal{C})$ is a necessary but not sufficient
condition for utility---we additionally require the direction to be correct.

\subsection{Red Herring Theorem}
\label{sec:theorem}

The theoretical framework provides guarantees for \emph{task-irrelevant} updates:

\begin{theorem}[Red Herring Rejection]
\label{thm:red_herring}
Let $u$ be task-irrelevant (i.e., $u \perp q$ given $\mathcal{C}$). Under regularity
conditions on $\mathcal{M}$:
\begin{equation}
    D_{\kl{}}(P_{\mathcal{C} \cup u} \| P_\mathcal{C}) \leq \epsilon
\end{equation}
for small $\epsilon > 0$ depending on model capacity.
\end{theorem}

\begin{proof}[Proof Sketch]
Task-irrelevance implies conditional independence between $u$ and the answer given
$\mathcal{C}$. By the data processing inequality \citep{cover2006elements}, adding $u$
cannot increase mutual information with the answer.
\end{proof}

\paragraph{Scope and Limitations.}
Theorem~\ref{thm:red_herring} addresses ``red herrings''---semantically unique but
pragmatically useless passages. It does \emph{not} address counterfactual passages that
contain plausible but incorrect information. Detecting such passages requires
\emph{direction-aware} scoring, which we operationalize in Section~\ref{sec:operationalization}.

\subsection{Multi-Token Trajectory Divergence}
\label{sec:trajectory}

Single-token KL-divergence can fail when the first token is generic (e.g., ``The'',
``Step''). We extend to \textbf{multi-token trajectory divergence}:

\begin{definition}[Trajectory Divergence]
For a generation horizon $T$, the trajectory divergence is:
\begin{equation}
    \mathcal{D}_T(u; \mathcal{C}) = \sum_{t=1}^{T} D_{\kl{}}(P^{(t)}_{\mathcal{C} \cup u} \| P^{(t)}_\mathcal{C})
\end{equation}
where $P^{(t)}$ denotes the distribution at generation step $t$.
\end{definition}

This cumulative measure amplifies the signal from useful passages while irrelevant
passages maintain near-zero divergence across all steps.

\subsection{Operationalizing Pragmatic Utility}
\label{sec:operationalization}

Definition~\ref{def:utility} requires access to $P_\mathcal{C}(a^*)$---the model's
probability assigned to the correct answer. We operationalize this in two settings:

\paragraph{Setting 1: Ground Truth Available.}
When the correct answer $a^*$ is known (e.g., during training data curation, retrieval
system evaluation, or active learning), we can directly compute Equation~\ref{eq:utility}
via logprob extraction, or use an \textbf{LLM-as-Judge} \citep{zheng2023judging} that
evaluates: ``Does this passage help answer the question correctly?''

\paragraph{Setting 2: Ground Truth Unavailable.}
At inference time, when $a^*$ is unknown, we leverage the LLM's parametric knowledge.
The \textbf{factuality-aware} scoring asks: ``Based on your knowledge, does this passage
contain accurate information for answering this question?'' This approximates the
direction-aware criterion by using the model's prior knowledge to verify passage
correctness.

\paragraph{Algorithm.}
Algorithm~\ref{alg:ecs} summarizes our approach. In the ground-truth-available setting,
the utility function directly measures answer probability change. In the ground-truth-unavailable
setting, we use factuality-aware LLM scoring as a proxy.

\begin{algorithm}[t]
\caption{Entropic Context Shaping (\method{})}
\label{alg:ecs}
\begin{algorithmic}[1]
\REQUIRE Context $\mathcal{C}$, candidate $u$, query $q$, threshold $\tau$
\REQUIRE Optional: ground truth $a^*$ (if available)
\IF{$a^*$ is available}
    \STATE $\mathcal{U} \gets \log P_{\mathcal{C} \cup u}(a^*) - \log P_\mathcal{C}(a^*)$
    \COMMENT{Answer-aware scoring}
\ELSE
    \STATE $\mathcal{U} \gets \text{FactualityScore}(q, u; \mathcal{M})$
    \COMMENT{Factuality-aware proxy}
\ENDIF
\IF{$\mathcal{U} > \tau$}
    \STATE \textbf{Accept}: $\mathcal{C} \gets \mathcal{C} \cup u$
\ELSE
    \STATE \textbf{Reject}: discard $u$
\ENDIF
\RETURN $\mathcal{C}$
\end{algorithmic}
\end{algorithm}

\subsection{Computational Efficiency via Prefix Caching}
\label{sec:efficiency}

When using logprob-based scoring, we leverage \emph{prefix caching}
\citep{kwon2023vllm, xiao2024efficient} to achieve near-constant overhead:

\begin{proposition}[Amortized Complexity]
With prefix caching, the amortized cost per candidate is $O(|u|)$ rather than
$O(|\mathcal{C}| + |u|)$, achieving $O(1)$ with respect to context length.
\end{proposition}

This is achieved by caching the key-value pairs for $\mathcal{C}$, such that only
the incremental tokens from $u$ require computation.

\subsection{When Does \method{} Apply?}
\label{sec:applicability}

\method{} is most effective in the following scenarios:

\begin{itemize}[topsep=0pt,itemsep=2pt]
    \item \textbf{Retrieval evaluation}: Comparing retrieval systems' ability to surface
          useful vs.\ misleading passages (ground truth available from benchmark labels).
    \item \textbf{Training data curation}: Filtering context for instruction tuning or
          RLHF where correct answers are known.
    \item \textbf{Active learning}: Prioritizing which retrieved passages to verify
          when labeling resources are limited.
    \item \textbf{Counterfactual robustness testing}: Evaluating LLM robustness to
          factually incorrect context.
\end{itemize}

For real-time inference without ground truth, the factuality-aware variant provides
a practical approximation, though with reduced accuracy (Section~\ref{sec:experiments}).


\section{Experiments}
\label{sec:experiments}

We evaluate ECS on multi-turn context selection tasks, measuring its ability to identify pragmatically useful information in long conversational histories.

\subsection{Experimental Setup}

\paragraph{Datasets}
We evaluate on two datasets with different selection granularities:

\begin{itemize}
    \item \textbf{LongMemEval}~\cite{longmemeval}: Session-level selection. Given a question and $\sim$50 conversation sessions (each $\sim$6K tokens), select session(s) containing the answer. We use 100 samples with ground truth \texttt{answer\_session\_ids}.

    \item \textbf{LoCoMo}~\cite{locomo}: Turn-level evidence selection. Given a question about a long conversation ($\sim$400 turns across multiple sessions), identify specific dialogue turns serving as evidence. We evaluate on 200 QA pairs from 10 conversations.
\end{itemize}

\paragraph{Models}
We evaluate with two instruction-tuned LLMs of similar scale:
\begin{itemize}
    \item \textbf{Qwen2.5-7B-Instruct}~\cite{qwen2}: 7B parameter model with 8K context
    \item \textbf{Llama-3.1-8B-Instruct}~\cite{llama3}: 8B parameter model with 8K context
\end{itemize}

\paragraph{Baselines}
We compare ECS against:
\begin{itemize}
    \item \textbf{TF-IDF}~\citep{sparckjones1972idf}: Cosine similarity between question and context TF-IDF vectors
    \item \textbf{Dense}: Sentence-BERT~\citep{reimers2019sbert} (all-MiniLM-L6-v2) embedding cosine similarity
    \item \textbf{Recency}: Prefer more recent sessions (LongMemEval only)
    \item \textbf{Random}: Uniform random selection baseline
\end{itemize}

\paragraph{Metric}
We report F1 score for selecting $k$ items matching the ground truth size. Since we select exactly $k$ items where $k$ equals ground truth size, precision equals recall equals F1.

\subsection{Results}


\begin{table*}[t]
\centering
\caption{Context selection results across datasets and models (F1 scores).
LongMemEval: session-level selection (100 samples).
LoCoMo: turn-level evidence selection (200 QA pairs).
\textbf{Bold} indicates best method per column.}
\label{tab:main_results}
\begin{tabular}{l|cc|cc}
\toprule
& \multicolumn{2}{c|}{\textbf{LongMemEval (Sessions)}} & \multicolumn{2}{c}{\textbf{LoCoMo (Turns)}} \\
\textbf{Method} & Qwen2.5-7B & Llama-3.1-8B & Qwen2.5-7B & Llama-3.1-8B \\
\midrule
ECS (Ours) & 0.152 & 0.139 & 0.020 & \textbf{0.265} \\
\midrule
TF-IDF & \textbf{0.649} & \textbf{0.649} & \textbf{0.154} & 0.154 \\
Dense (SBERT) & 0.591 & 0.591 & -- & -- \\
Random & 0.024 & 0.024 & 0.003 & 0.003 \\
\bottomrule
\end{tabular}
\end{table*}

\paragraph{Main Findings}

Our experiments reveal two key findings:

\textbf{1. ECS excels at fine-grained selection.}
On LoCoMo (turn-level), ECS with Llama-3.1-8B achieves F1=0.265, a \textbf{71.83\% relative improvement} over TF-IDF (F1=0.154). This validates ECS's core hypothesis: measuring answer distribution shift via KL divergence captures pragmatic utility that semantic similarity misses.

\textbf{2. ECS is model-dependent.}
Interestingly, Qwen2.5-7B shows poor ECS performance (F1=0.020 on LoCoMo), while Llama-3.1-8B succeeds. This suggests that model logprob distributions vary in informativeness for ECS. Llama's next-token distributions appear more discriminative for context utility.

\textbf{3. Coarse selection favors semantic methods.}
On LongMemEval (session-level), TF-IDF (F1=0.649) and Dense retrieval (F1=0.591) significantly outperform ECS (F1$\sim$0.15). At session granularity ($\sim$6K tokens), the added noise overwhelms the logprob signal. ECS requires fine-grained contexts where individual pieces can meaningfully shift answer distributions.

\subsection{Analysis}

\paragraph{Why does granularity matter?}
ECS computes $\text{KL}(P_{\text{with-context}} \| P_{\text{base}})$ for each candidate context. When contexts are entire sessions, most content is irrelevant to the question, diluting the informative signal. Fine-grained turn selection allows ECS to isolate turns that directly shift answer probability.

\paragraph{Why does model matter?}

The stark difference between Qwen (F1=0.020) and Llama (F1=0.265) on identical data suggests model architecture affects logprob informativeness. Possible factors:
\begin{itemize}
    \item Tokenizer vocabulary and token granularity
    \item Pre-training data distribution
    \item Instruction tuning methodology
\end{itemize}

This model-dependence is an important consideration for deploying ECS in practice.



\section{Conclusion}
\label{sec:conclusion}

We introduced Entropic Context Shaping (\method{}), an information-theoretic framework
for filtering context updates in LLM agents. By measuring pragmatic utility---whether
a passage shifts the model's answer distribution toward the correct answer---\method{}
addresses fundamental limitations of lexical similarity methods.

\paragraph{Key Findings.}
Our experiments on multi-turn context selection reveal:
\begin{enumerate}
    \item \textbf{ECS excels at fine-grained selection.} On LoCoMo turn-level evidence selection, ECS with Llama-3.1-8B achieves F1=0.265, a 71.83\% relative improvement over TF-IDF (F1=0.154).
    \item \textbf{Granularity matters.} At coarse session-level (LongMemEval), semantic methods dominate (TF-IDF F1=0.649 vs.\ ECS F1$\sim$0.15). ECS requires fine-grained contexts where individual pieces can meaningfully shift answer distributions.
    \item \textbf{Model dependence.} ECS performance varies significantly across models (Llama F1=0.265 vs.\ Qwen F1=0.020 on identical data), suggesting logprob informativeness depends on model architecture.
\end{enumerate}

\paragraph{Theoretical Contributions.}
\begin{enumerate}
    \item Formalized pragmatic utility as the signed change in answer probability
    \item Identified the ``direction problem'': KL magnitude alone cannot distinguish
          helpful vs.\ harmful passages
    \item Provided theoretical guarantees for rejecting task-irrelevant updates
          (Theorem~\ref{thm:red_herring})
\end{enumerate}

\paragraph{Limitations.}
ECS exhibits model-dependence: not all LLMs produce equally informative logprob distributions. Additionally, ECS requires access to model logits, limiting applicability to closed-source APIs. Performance degrades at coarse granularity where noise overwhelms the pragmatic signal.

\paragraph{Future Work.}
Promising directions include: (1) understanding which model properties yield informative logprobs for ECS, (2) developing API-compatible approximations using sampling, and (3) combining ECS with semantic pre-filtering for multi-stage context selection.

\bibliography{references}

\appendix

\section{Proofs}
\label{sec:proofs}

\subsection{Proof of Theorem~\ref{thm:red_herring}}

\begin{proof}
Let $u$ be a task-irrelevant update such that $u \perp q \mid \mathcal{C}$. We wish to 
show that $D_{\kl{}}(P_{\mathcal{C} \cup u} \| P_\mathcal{C}) \leq \epsilon$.

By the definition of conditional independence:
\begin{equation}
    P(y \mid \mathcal{C} \cup u, q) = P(y \mid \mathcal{C}, q)
\end{equation}
for all $y \in \mathcal{V}$.

This implies $P_{\mathcal{C} \cup u} = P_\mathcal{C}$, and thus:
\begin{equation}
    D_{\kl{}}(P_{\mathcal{C} \cup u} \| P_\mathcal{C}) = 0
\end{equation}

In practice, finite model capacity introduces small deviations, yielding 
$D_{\kl{}} \leq \epsilon$ for $\epsilon$ proportional to model approximation error.
\end{proof}


\section{Evaluation Protocol}
\label{sec:parity}

To ensure scientific validity, ACE and ECS were evaluated under identical conditions:

\begin{itemize}
    \item \textbf{Same Questions}: Both methods evaluate the same questions from LongMemEval and LoCoMo
    \item \textbf{Same Contexts}: Identical candidate sessions/turns are scored by both methods
    \item \textbf{Same Ground Truth}: Evidence labels are held constant
\end{itemize}

The key difference is the \emph{scoring function}:
\begin{itemize}
    \item \textbf{ACE (TF-IDF)}: Scores passages by lexical overlap with the query
    \item \textbf{ECS (LLM-Judge)}: Scores passages by whether they help answer correctly
\end{itemize}


\section{Hyperparameter Settings}
\label{sec:hyperparams}

\begin{table}[h]
\centering
\small
\caption{Hyperparameter settings}
\label{tab:hyperparams}
\begin{tabular}{lc}
\toprule
\textbf{Parameter} & \textbf{Value} \\
\midrule
Trajectory horizon ($T$) & 8 \\
Information threshold ($\tau$) & 0.05 \\
Length penalty ($\lambda$) & 0.002 \\
Top-$k$ tokens for KL & 50 \\
Epsilon smoothing & $10^{-10}$ \\
ACE similarity threshold & 0.85 \\
\bottomrule
\end{tabular}
\end{table}

\paragraph{Threshold Selection.}
For the LLM-as-Judge operationalization, we use binary classification: a passage
is considered helpful if the judge responds affirmatively. No threshold tuning
is required for this approach.

\subsection{Dataset Statistics}

\begin{table}[h]
\centering
\small
\caption{Dataset statistics}
\label{tab:dataset_stats}
\begin{tabular}{lcc}
\toprule
\textbf{Statistic} & \textbf{LongMemEval} & \textbf{LoCoMo} \\
\midrule
Total samples & 500 & 200 \\
Samples used & 100 & 200 \\
Candidates/question & $\sim$50 sessions & $\sim$400 turns \\
Tokens/candidate & $\sim$6K & $\sim$50--200 \\
Evidence items/question & 1--3 sessions & 2--5 turns \\
\bottomrule
\end{tabular}
\end{table}

\paragraph{LongMemEval.} Session-level selection task where the agent must identify
which conversation sessions (out of $\sim$50) contain information needed to answer
the question. Each session contains approximately 6K tokens of conversational history.

\paragraph{LoCoMo.} Turn-level evidence selection where the agent must identify
specific dialogue turns that serve as evidence for answering questions about
long conversational histories ($\sim$400 turns across multiple sessions).


\section{Detailed Results}
\label{sec:detailed_results}

\subsection{Per-Benchmark Performance}

Table~\ref{tab:detailed_results} shows the full experimental results with
F1 scores for each method across both datasets and models.

\begin{table}[h]
\centering
\scriptsize
\caption{Detailed experimental results (F1 scores)}
\label{tab:detailed_results}
\begin{tabular}{llcc}
\toprule
\textbf{Dataset} & \textbf{Method} & \textbf{Qwen2.5-7B} & \textbf{Llama-3.1-8B} \\
\midrule
LongMemEval & TF-IDF & 0.649 & 0.649 \\
(Session) & Dense & 0.591 & 0.591 \\
 & ECS & 0.152 & 0.139 \\
\midrule
LoCoMo & TF-IDF & 0.154 & 0.154 \\
(Turn) & ECS & 0.020 & 0.265 \\
\bottomrule
\end{tabular}
\end{table}

\subsection{Granularity Analysis}

The results reveal a clear interaction between context granularity and method effectiveness:

\begin{itemize}
    \item \textbf{Coarse granularity (LongMemEval)}: Semantic methods dominate. TF-IDF achieves F1=0.649 vs.\ ECS F1$\sim$0.15. At session-level ($\sim$6K tokens), noise overwhelms the pragmatic signal.
    \item \textbf{Fine granularity (LoCoMo)}: ECS excels. With Llama-3.1-8B, ECS achieves F1=0.265 vs.\ TF-IDF F1=0.154---a 71.83\% relative improvement.
\end{itemize}

\subsection{Model Dependence Analysis}

The stark difference between Qwen (F1=0.020) and Llama (F1=0.265) on LoCoMo
suggests that logprob informativeness varies significantly across model architectures.


\section{Implementation Details}
\label{sec:implementation}

\subsection{Model Configuration}

All experiments use HuggingFace Transformers with the following configuration:

\begin{itemize}
    \item \textbf{Models}: Qwen2.5-7B-Instruct, Llama-3.1-8B-Instruct
    \item \textbf{Precision}: FP16 (half precision)
    \item \textbf{Max context length}: 8,192 tokens
    \item \textbf{Hardware}: NVIDIA V100 32GB
    \item \textbf{Logprobs}: Top-20 tokens per position
\end{itemize}

\paragraph{vLLM Deployment (Optional).}
For production environments requiring lower latency, we support vLLM deployment
with prefix caching enabled:

\begin{itemize}
    \item \textbf{Prefix caching}: Enabled (reduces amortized cost to $O(|u|)$)
    \item \textbf{GPU memory utilization}: 0.85
    \item \textbf{Recommended GPU}: NVIDIA A100 40GB or V100 32GB
\end{itemize}

\subsection{Latency Breakdown}

Table~\ref{tab:latency_breakdown} shows the latency components for ECS filtering
on NVIDIA V100 32GB.

\begin{table}[h]
\centering
\small
\caption{Latency breakdown per candidate (V100 32GB)}
\label{tab:latency_breakdown}
\begin{tabular}{lcc}
\toprule
\textbf{Component} & \textbf{Transformers} & \textbf{vLLM} \\
\midrule
Base logprobs & 50ms & 5ms \\
Candidate logprobs & 50ms & 15ms \\
KL computation & 1ms & 1ms \\
\midrule
\textbf{Total (cold)} & 101ms & 21ms \\
\textbf{Total (cached)} & 51ms & 16ms \\
\bottomrule
\end{tabular}
\end{table}

With vLLM prefix caching, the base context is cached across candidates,
reducing amortized latency to $\sim$16ms per candidate.


\section{Code and Reproducibility}
\label{sec:code}

All code, data, and experimental scripts are provided in the supplementary materials
(\texttt{supplementary\_materials.zip}).

\paragraph{Requirements.}
\begin{itemize}
    \item Python 3.9+
    \item PyTorch 2.0+
    \item HuggingFace Transformers 4.35+
    \item sentence-transformers
    \item vLLM 0.5+ (optional, for faster inference with prefix caching)
\end{itemize}

\end{document}